\documentclass{article} 
\usepackage{iclr2024_conference,times}

\usepackage{amsmath,amsfonts,bm}

\def\eqref#1{equation~\ref{#1}}

\def\1{\bm{1}}

\DeclareMathAlphabet{\mathsfit}{\encodingdefault}{\sfdefault}{m}{sl}
\SetMathAlphabet{\mathsfit}{bold}{\encodingdefault}{\sfdefault}{bx}{n}

\usepackage{hyperref}
\usepackage{url}

\usepackage{algorithm}
\usepackage{algorithmic}
\usepackage{amsthm}
\newtheorem{theorem}{Theorem}
\newtheorem{lemma}[theorem]{Lemma}
\usepackage{caption}
\usepackage{subcaption}
\usepackage{mathtools}
\usepackage{adjustbox}
\usepackage{multirow}
\usepackage{siunitx}

\usepackage{enumitem}
\setlist{leftmargin=3mm}

\usepackage{wrapfig,lipsum,booktabs}
\usepackage{needspace}

\title{Controllable Data Generation Via \\ Iterative Data-Property Mutual Mappings}

\author{%
    Bo Pan\textsuperscript{1}, Muran Qin\textsuperscript{2}, Shiyu Wang\textsuperscript{1}, Yifei Zhang\textsuperscript{1}, Liang Zhao\textsuperscript{1} \\
    \textsuperscript{1}Emory University, Atlanta, USA \\ \textsuperscript{2}University of California, San Diego, San Diego, USA \\
    \{bo.pan, shiyu.wang, yifei.zhang2, liang.zhao\}@emory.edu, muqin@ucsd.edu \\
}

\iclrfinalcopy 
\begin{document}

\maketitle

\begin{abstract}
Deep generative models have been widely used for their ability to generate realistic data samples in various areas, such as images, molecules, text, and speech. One major goal of data generation is controllability, namely to generate new data with desired properties. Despite growing interest in the area of controllable generation, significant challenges still remain, including 1) disentangling desired properties with unrelated latent variables, 2) out-of-distribution property control, and 3) objective optimization for out-of-distribution property control. To address these challenges, in this paper, we propose a general framework to enhance VAE-based data generators with property controllability and ensure disentanglement. Our proposed objective can be optimized on both data seen and unseen in the training set. We propose a training procedure to train the objective in a semi-supervised manner by iteratively conducting mutual mappings between the data and properties. The proposed framework is implemented on four VAE-based controllable generators to evaluate its performance on property error, disentanglement, generation quality, and training time. The results indicate that our proposed framework enables more precise control over the properties of generated samples in a short training time, ensuring the disentanglement and keeping the validity of the generated samples.
\end{abstract}

\section{Introduction}
Deep generative models, such as Variational Autoencoders (VAEs) \cite{kingma2013auto, oussidi2018deep}, Generative Adversarial Networks (GANs) \cite{goodfellow2020generative, guo2022systematic}, normalizing flows \cite{dinh2016density, rezende2015variational}, and diffusion models \cite{ho2020denoising}, have become increasingly popular in recent years for their ability to model the underlying distribution of data, which is useful for tasks such as data generation, representation learning, and anomaly detection. Specifically, deep generative models have emerged as a powerful tool for generating realistic data samples with complex structures in various domains, such as images, molecules, and natural language.
However, in many real-world applications, it is necessary to generate data with specific properties \cite{wang2022controllable}. For example, in the field of chemistry, scientists often aim to design molecules with desired characteristics, such as high binding affinity, low toxicity, and specific reactivity. This presents a challenge as generating data with specific properties is a non-trivial task.

The challenge of generating data with specified properties is usually addressed through disentangled representation learning with generative models. This type of method aims to learn latent variables that can separate out the independent factors of variation that contribute to the observed data. With a low-dimensional bottleneck, VAE models naturally have an advantage in learning a disentangled latent space.
Particularly, several works have investigated the application of disentangled learning in the generation domain with VAE models: CSVAE \cite{csvae} transfers image attributes by correlating latent variables with desired properties; Semi-VAE \cite{semivae} pairs latent space with properties by minimizing the mean-square-error (MSE) between latent variables and desired properties; PCVAE \cite{pcvae} synthesizes image objects with desired positions and scales by enforcing the mutual dependence between disentangled latent variables and properties; CorrVAE \cite{wang2022controllable} uses an explainable mask pooling layer to precisely preserve the correlated mutual dependence between latent vectors and properties.

Although the current results on this topic have been encouraging, there are three significant challenges that are still not well solved: (1) How to ensure the precision of property control? (2) How to ensure the disentanglement between the latent variables and the properties they do not control? (3) How to ensure the above disentanglement and controllability are guaranteed in general, instead of merely in the training domain?

To address the above challenges, we propose a general framework with customizable constraints that allow it to be applied to various VAE-based controllable generators. We specify two of the constraints that can enforce the disentanglement between the property to control and the unrelated latent variables. We further propose an overall training objective of this framework that can directly minimize the property errors and also cover data both seen and unseen in the training set. To optimize this training objective effectively, we propose a training procedure that optimizes the model with both data and properties, including those outside the training set. The contributions of this paper are summarized as follows:

\begin{itemize}
\item A general controllable generation framework with enforced disentanglement. We propose a general framework with enhanced property control and disentanglement between properties of interest and unrelated latent variables. The proposed framework can be incorporated with existing VAE-based property controllable generation models.
\item An objective covering property values out of the range of the training set. We extend the traditional objective function for property controllable generation to encompass data both within and outside the distribution of the training set, enabling the adaptation of model to out-of-distribution data ranges.
\item A training procedure to effectively optimize the objective by iteratively mapping data and properties. We present a procedure for training the model with both data and properties, including those unseen in the training set, by iteratively conducting mutual mappings between data and properties. 
\item A comprehensive evaluation of the proposed framework including quantitative and qualitative results. The results demonstrate that our framework can significantly enhance the precision of property control and disentanglement with a shorter training time compared to training the base models.
\end{itemize}

\section{Related Work}

Controllable deep data generation aims to generate data with desired properties \cite{wang2022controllable}. It has various critical applications including molecule design \cite{jin2020multi}, drug discovery \cite{pan2022property}, image editing \cite{deng2020disentangled}, and speech synthesis \cite{henter2018deep, keskar2019ctrl}. 
With the disentangled representation learning ability of VAEs, a lot of disentangled VAE-based generators with controllability have been proposed, forming an important stream of general controllable generation methods. CondVAE \cite{condvae1, condvae2} learns a structured latent space by conditioning the latent representation on the property. CSVAE \cite{csvae}, Semi-VAE \cite{semivae}, and PCVAE \cite{pcvae} learn mutual dependence between properties and latent variables and manipulate the value of latent variables to control the generation process. The most recent CorrVAE \cite{wang2022multi} handles the correlation of desired properties 
during controllable generation. It recovers semantics and the correlation of properties through disentangled latent vectors learned via an explainable mask pooling layer.

\section{Problem Formulation} \label{pf}

Given a dataset $\mathcal{D}=\{\mathcal{X}, \mathcal{Y}\}$ consisting of elements $(x, y)$, where $x\in \mathbb{R}^n$ and $y= \{y_k \in \mathbb{R}\}^K_{k=1}$ representing $K$ properties of interest and $y=\mathbf{f}(x)$. The task of controllable data generation is to learn a generative model that can generate $x$ with the desired property value $y$. 

This problem is usually solved with variation Bayesian-based methods \cite{condvae1,condvae2,semivae,csvae,pcvae,wang2022multi}. These methods share a pattern that relates a group of latent variables to the properties to control. Specifically, they split the latent variables of generative model into two groups, $z$ and $w$, where the variable $w$ controls the properties of interest labeled by $y$, and $z$ controls all other aspects of $x$. During the controllable generation phase, they sample $z$ from the prior distribution, and $w$ is attained with $w=\mathbf m(y,z)$ with desired property $y$ and sampled $z$, where $\mathbf m$ is a mapping from desired properties to $w$.\footnote{Notably, some of the work \cite{semivae, pcvae} does not require $z$ to calculate $w$, namely $w=\mathbf m(y)$. Here we write it as a form that generally for different methods, $w=\mathbf m(y,z)$.} Then the data can be generated with the generative model using the latent factors $z$ and $w$.

Although the current results on this controllable generation problem have been encouraging, there are three significant challenges that are still not well-solved, including:

\textbf{Challenge 1}: \textbf{Difficulties in precisely controlling the properties of generated data.} Traditional methods ensure property control by maximizing the mutual information between a group of latent variables and the properties and minimizing the mutual information between this group of latent variables and the others. This yields an inaccuracy in property control since the bias of these two processes combine to affect the performance. 

\textbf{Challenge 2}: \textbf{Difficulties in ensuring the independence between properties to control and latent variables unrelated to these properties.} It is vital to maintain this independence during generation to accurately control specific properties while ensuring other aspects of the generated data remain unaffected. However, current methods mostly enforce independence between two groups of latent variables that are related and unrelated to the properties, which is an indirect way since it focuses on the latent factors related to the properties instead of the property value itself. For example, CSVAE \cite{csvae} minimizes the conditional entropy between two sets of latent variables (those related and unrelated to controlled properties), while PCVAE \cite{pcvae} employs the total correlation term to encourage independence between these variable groups. Directly ensuring independence between property values and latent variables remains a significant challenge.

\textbf{Challenge 3}:\textbf{ Difficulties in controlling the properties in out-of-distribution ranges. }Deep models are limited when the desired property values are out of the distribution in the dataset due to the scarcity of training data. This makes it difficult to generate data with desired property values in out-of-distribution ranges.

\textbf{The goal of this paper:} We aim to propose a general framework that can enhance the current controllable data generation techniques to better overcome the above three challenges, which are detailed in the following sections.
\section{Methodologies}

In Section~\ref{sec:gf}, we propose a generic framework that covers various VAE-based controllable data generators. Then we develop the learning objective for enhancing the controllability and disentanglement of our framework inside and outside the training domain in Section~\ref{sec:ood}. Finally, we develop a new algorithm that solves the learning objective by closing the loop of data-to-property and property-to-data mappings in the last subsection.
\subsection{A General Framework for VAE-Based Property Controllable Data Generation}
\label{sec:gf}

In this section, we propose a generic framework that can enhance existing controllable data generators, which are typically variational Bayesian techniques, as special cases. Our framework has two key components: (1) \textbf{An enhanced shared backbone}: We enhance the backbone shared by different existing approaches by a) a training objective that directly minimizes the property error, and b) an important independence assumption required by controllable data generation but overlooked by some of the existing works, which allows a better disentanglement between the desired properties $y$ and unrelated latent variables $z$. (2) \textbf{A customizable module for adapting unshared assumptions:} Different variational Bayesian-based controllable generators are typically differentiated by their respective assumptions used to derive their learning objective for model inference. We propose a module to fit the different assumptions of these existing models. We introduce them respectively in the following.

\noindent 1) \textbf{An enhanced shared backbone}.
Given an approximate posterior $q_{\phi}({z}, {w} |{x}, {y})$, we start with Jensen's inequality to obtain the variational lower bound
\begin{equation} \label{var_lower_bound}
\begin{aligned}
 \log p_{\theta, \gamma}({x}, {y})
&=\log \mathbb{E}_{q_\phi({z}, {w}|{x}, {y})}\left[p_{\theta, \gamma}({x}, {y}, {w}, {z}) / q_\phi({z}, {w}|{x}, {y})\right] \\
&\geq  \mathbb{E}_{q_\phi({z}, {w} |{x}, {y})}\left[\log p_{\theta, \gamma}({x}, {y}, {w}, {z}) / q_\phi({z}, {w}| {x}, {y})\right]
\end{aligned}
\end{equation}

\begin{wrapfigure}[8]{r}{0.27\textwidth}  
\vspace{-9mm}
\begin{center}
        \includegraphics[width=0.16\textwidth]{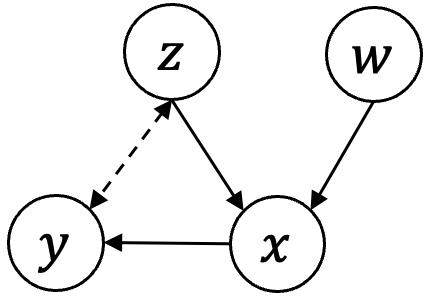}
    \end{center}\vspace{-0.4cm}
    \caption{Graphical model of the generative process. The dotted line indicates enforced independence. \vspace{-0.4cm}}
    \label{fig:graphical}
\end{wrapfigure}

To accomplish the aforementioned goal of controllable data generation, we employ the assumption of $z \perp y$: the latent variable $z$ is assumed to be independent of the properties of generated data ($y$). In other words, $z$ controls aspects of generated data other than the properties of interest. With this assumption, the graphical model can be shown as Fig.~\ref{fig:graphical}. The joint log-likelihood in Equation \ref{var_lower_bound} can be then derived as:
\begin{equation} \label{joint_log_likelihood}
\begin{aligned}
    \log p_{\theta, \gamma}({x}, {y}, {w}, {z})
    &=\log p_{\theta, \gamma}(x, y|w, z)+\log p(w, z) \\
    &=\log p(y|x, w, z) + \log p_{\theta}(x|w, z)+\log p(w, z) \\
    &=\log p_{f}(y|x) + \log p_{\theta}(x|w, z)+\log p(w, z) \qquad \text{(due to D-separation)}
\end{aligned}
\end{equation} 
where $p_\theta$ is the conditional generative process parameterized by $\theta$, and $p_f$ is the conditional distribution of measuring $y$ with $x$, parameterized by $\mathbf f(\cdot)$. This joint log-likelihood is universal across various VAE-based controllable generators, while enhanced with an important assumption $y\perp z$.

2) \textbf{A customizable module for adapting unshared assumptions.} By incorporating the joint log-likelihood defined in Eq.~\ref{joint_log_likelihood} into the variational lower bound term in Eq.~\ref{var_lower_bound}, we obtain the negative part as an upper bound on $-\log p_{\theta, \gamma}(x, y)$, which serves as the overall objective for our framework. Specifically, different variational Bayesian models are differentiated from each other by the assumptions they use. To make our framework easily adapt to their particular assumptions, we propose a  module $\mathbf C_i(z,w,x,y),  \ i=0, 1,2,...,N_c$ to entail the additional assumption(s) as constraint(s) in the objective function. Then the overall objective can be written as
\begin{equation} \label{loss_0}
\begin{aligned}
    \text{Minimize}\quad 
    &-\mathbb{E}_{q_\phi(z, w|x, y)}[\log p_{\theta}(x|z, w)]\\
    &-\mathbb{E}_{q_\phi(z, w|x, y)}[\log p(y|x)]\\
    &+D_{K L}(q_\phi(z, w|x, y) \| p(z, w)) \\
    \text{Subject to}\quad  &\ \mathbf C_i(z,w,x,y)=0,  \quad\forall i\in\{0, 1,2,...,N_c\}       
\end{aligned}
\end{equation}
where $D_{KL}$ denotes the Kullback–Leibler divergence. Each constraint $C_i(z,w,x,y)$ can be a variable that can be customized to different assumptions of the specific generators, and $N_c$ is the total number of constraints. With this framework, our assumption $y\perp z$ can be enforced as specified constraints to the objective in Equation \ref{loss_0} as: 
\begin{equation} \label{constraints}
\begin{aligned}
\mathbf C_0(z,w,x,y)=p(y,z)-p(y)p(z)=0
\end{aligned}
\end{equation}
which represents the independence of the variables $y$ and $z$. So that the first constraint is specified in our framework, and the remaining are left to be customizable to different base generators.

In Appx.~\ref{proof:base}, we give examples to show that various existing variational Bayesian approaches for controllable data generation can be incorporated with our framework by adding customized constraints $C_i(z,w,x,y)=0$.

\begin{figure}[t]  
\centering  
\includegraphics[width=0.7\textwidth,trim=140 130 200 300,clip]{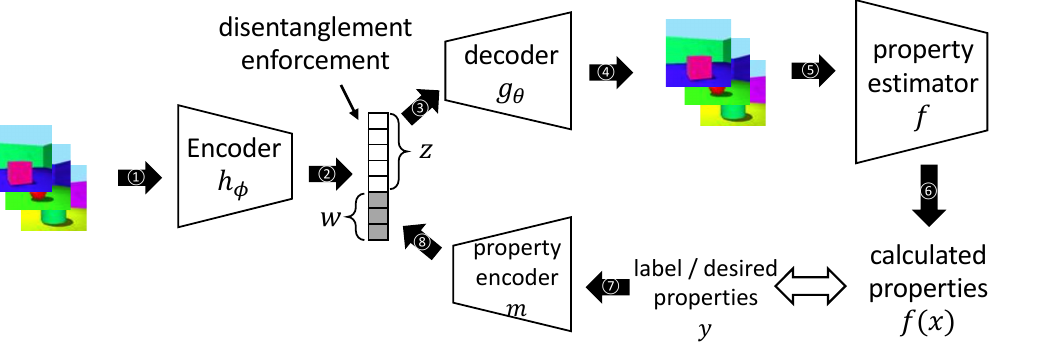}  \caption{Illustration of our proposed framework. When optimized with data in the training set, the parts \textcircled{1}\textcircled{2}\textcircled{7}\textcircled{8}\textcircled{3} run with pairs of $(x,y)$, which is the same as base VAE generators with a encoder, decoder and property encoder, and here $y$ denotes the label of $x$. When optimized with data out of the training set, $y$ is picked as the desired property value, which is used to infer $w$; then the generated data is decoded and measured with $f$ to get the actual property $f(x)$, which is used to calculate a property loss with $y$, thus all \textcircled{1} to \textcircled{8} run. \protect\footnotemark} 
\label{modelfig}
\end{figure}
\textbf{Model Architecture.} The objective above requires three components that should be modeled with neural networks: 
\begin{itemize}
    \item Encoder \(\mathbf h_\phi\), to model \(q_\phi(z,w|x,y)\) by $(z,w)=\mathbf h_\phi(x,y)$.
    \item Decoder \(\mathbf g_\theta\), to model \(p_\theta(x|z,w)\) by $x=\mathbf g_\theta(z,w)$.
    \item Property encoder \(\mathbf m_\gamma\), to map desired $y$ \footnote{$z$ is optional depending on the specific base generators. This is explained in detail in Sec.~\ref{pf}.} to $w$ by $\mathbf m_\gamma(y,z)$.
\end{itemize}

The implementation of the encoder and decoder can vary with the data type of \(x\), such as images and graphs. These model components can also be viewed from the illustration of our training procedure (presented in the next section) in Fig.~\ref{modelfig}.

\subsection{Out-of-Distribution Property Control}
\label{sec:ood}
\textbf{Overview.} We aim to address situations where the desired property values for generating new data are out of the range of the training dataset. For instance, in the field of chemistry, one might aim to produce a new molecule with a higher molecular mass that shares similar chemical patterns with molecules from the training set with lower masses.
In this section, we extend our framework to accommodate this out-of-distribution generation scenario by extending the objective to both seen and unseen data. The derivation is organized by extending the three terms of  Eq.~\ref{loss_0} to get $\mathcal{L}_1$, $\mathcal{L}_2$, and $\mathcal{L}_3$, correspondingly.

\textbf{Notations.} We use the notations $\{\mathcal{X}_1, \mathcal{Y}_1\}$ to denote the data and labels in the dataset $\mathcal{D}$, and $\{\mathcal{X}_2, \mathcal{Y}_2\}$ to denote the data that are not present in $\mathcal{D}$. In the following derivations, we decompose each term of Eq.~\ref{loss_0} to induce the overall objective that can be optimized in both $\{\mathcal{X}_1, \mathcal{Y}_1\}$ and $\{\mathcal{X}_2, \mathcal{Y}_2\}$ domains.

\textbf{Derivation of $\mathcal{L}_1$}: For the first term of Eq. \ref{loss_0}, we use a Gaussian distribution with mean of $\mathbf g_\theta(z, w)$ and standard deviation of $\sigma_p$ to approximate the conditional distribution $p_\theta(x|z,w)$, i.e. $ p_\theta(x|z,w)\rightarrow \mathcal{N}(x|\mathbf g_\theta(z, w),\sigma^2_p)$, then we have
\begin{equation} \label{l1_0}
\begin{aligned}
\underset{\theta, \phi}{\mathrm{argmin}}-\mathbb{E}_{q_\phi(z, w|x, y)}[\log p_\theta(x|z,w)] 
=\ & \underset{\theta, \phi}{\mathrm{argmin}}-\mathbb{E}_{q_\phi(z, w|x, y)}-\log \mathcal{N}(x|\mathbf {g_\theta}(z, w),\sigma^2_p) \\
=\ & \underset{\theta, \phi}{\mathrm{argmin}}-\mathbb{E}_{q_\phi(z, w|x, y)}-(x-\mathbf {g_\theta}(z, w))^2
\end{aligned}
\end{equation}
where we use $\mathbf g_\theta(z, w)$ to denote the function that models the conditional distribution $p_\theta(x|z,w)$, such that of $\hat x=\mathbf g_\theta(z, w)$. 

In the above equation, the latent variables $z,w$ are inferred from $x,y$ with $q_\phi(z,w|x,y)$. Here we use $\mathbf h_\phi(x,y)$ to denote a function that models the distribution $q_\phi(z,w|x,y)$ so that $(z, w)=\mathbf h_\phi(x,y)$. 
We employ a summative representation of Eq.~\ref{l1_0} across all data points, which constitutes our first loss term as 
\begin{equation} \label{l1_1}
\begin{aligned}
\mathcal{L}_1&= \sum_{x,y\in\mathcal{X}_1\cup\mathcal{X}_2,\mathcal{Y}_1\cup \mathcal{Y}_2} (x-\mathbf g_\theta(\mathbf h_\phi( x, y)))^2 
\end{aligned}
\end{equation}
Penalizing this term encourages that the data can be recovered from the latent factors extracted from it. 

\textbf{Derivation of $\mathcal{L}_2$}: For the second term in Eq.~\ref{loss_0}, i.e., $-\mathbb{E}_{q_\phi(z, w|x, y)}[\log p_{}(y|x)]$, the distribution $p(y|x)$ is given by the applying the law to calculate the actual property values of the generated data, i.e. $\mathbf f(x)$.
We approximate the conditional distribution  
$p(y|x)$ as a Gaussian distribution with mean $y$ and standard deviation $\sigma$. We aim to optimize the mean to actual property $\mathbf f(x)$  to ensure precise property control. So we derive this term as:
\begin{equation} \label{l2_0}
\begin{aligned}
\underset{\theta, \phi, \gamma}{\mathrm{argmin}}-\mathbb{E}_{q_\phi(z, w|x, y)}\log p(y|x)&= \underset{\theta, \phi, \gamma}{\mathrm{argmin}}-\mathbb{E}_{q_\phi(z, w|x, y)}-\log{\mathcal{N}(y|\mathbf f(x),\sigma^2)} \\
&=\underset{\theta, \phi, \gamma}{\mathrm{argmin}}\ \mathbb{E}_{q_\phi(z, w|x, y)}(y-\mathbf f(x))^2 \\
\end{aligned}
\end{equation}

where $\mathcal{N}$ denotes Gaussian distribution. We also employ a summation form across all seen and unseen data of Eq.~\ref{l2_0} to be our second loss term as:
\begin{equation}\label{l2_1}
\begin{aligned}
\mathcal{L}_2&=
\sum_{x,y\in\mathcal{X}_1\cup \mathcal{X}_2,\mathcal{Y}_1\cup \mathcal{Y}_2}(y-\mathbf f( x))^2 \\
&=\sum_{x,y\in\mathcal{X}_1,\mathcal{Y}_1}(y-\mathbf f(\mathbf g_\theta(\mathbf h_\phi(x,y)))^2+\sum_{x,y\in\mathcal{X}_2,\mathcal{Y}_2}(y-\mathbf f(x))^2
\end{aligned}
\end{equation}
This term aims to encourage the calculated property value of generated data, $\mathbf f(x)$, to be close to the given value $y$.

\textbf{Derivation of $\mathcal{L}_3$}: We use the third term in Eq. \ref{loss_0} as another part of our objective
\begin{equation} 
\begin{aligned}
\mathcal{L}_3&=D_{K L}(q_\phi(z, w|x, y) \| p(z, w))
\end{aligned}
\end{equation}
where $D_{KL}$ denotes the Kullback–Leibler divergence.

\textbf{Overall Objective.} Then the objective over all seen and unseen data can be represented as:
\begin{equation}\label{overallobjective}
\begin{aligned}
\text{Minimize}\quad\quad\quad \quad\quad \quad  \mathcal{L}_1+\alpha\mathcal{L}_2&+\beta\mathcal{L}_3 \\
\text{Subject to}\quad  \ \mathbf C_i(z,w,x,y)=0, \quad &\forall x,y\in\mathcal{X}_1\cup \mathcal{X}_2,\mathcal{Y}_1\cup \mathcal{Y}_2 \\
&\forall i\in\{0, 1,2,...,N_c\} 
\end{aligned}
\end{equation}  
 where $\alpha$ and $\beta$ are weights of different terms, and one of the constraints, $\mathbf C_0$, has been specified in Eq. \ref{constraints} as $\mathbf C_0(z,w,x,y)=p(y,z)-p(y)p(z)$.

\subsection{Training Procedure}

\textbf{Overview.} The above training objective is hard to optimize directly due to several reasons: (1) The data in $\mathcal{X}_2, \mathcal{Y}_2$ domain is unseen, thus the corresponding terms cannot be trained; (2) The constraint $y\perp z$ is hard to be directly enforced during training; (3) An effective strategy to organize the training is required. In this section, we will solve these three problems by sampling additional data, converting the constraint to a loss term, and proposing a training strategy to effectively optimize the objective.

Therefore, we have devised a training procedure to optimize the model by generating the $x$ that corresponds to the desired property $y$. 

\textbf{1) Training with unseen data.} The power of controllable VAE models allows us to generate additional data with desired properties for training. Thus we can use the decoder to generate additional data by sampling the property values in our preferred distribution and sampling the unrelated latent variables for data diversity. Specifically, the additional data $x \in \mathcal{X}_2$ is generated by $x=\mathbf{g}_\theta(z,w)$, where $z$ is sampled by $z\sim p(z)$, and $w$ is attained by $w=\mathbf{m}_\gamma(y,z)$, and $y$ is sampled by $y\sim p_2(y)$, where $p_2(y)$ is the user-preferred distribution in sampling new data. $p_2(y)$ can be out of the distribution of $y$ in $\mathcal{Y}_1$ domain, allowing the out-of-distribution training. Then a loss term in the form of a summation on the $\mathcal{X}_2, \mathcal{Y}_2$ can be converted as:
\begin{equation}
    \sum_{x,y\in\mathcal{X}_2,\mathcal{Y}_2}\ell(x,y)=n_{sample}\cdot \mathbb{E}_{p_2(y)}\mathbb{E}_{p(z)}\ell(\mathbf{g}_\theta(z,\mathbf{m}_\gamma(y,z)),y)
\end{equation}
where $\ell$ denoted any loss function calculated with $x$ and $y$, $n_{sample}$ is the total number of generated new data samples.
Converting terms summed on $\mathcal{X}_2, \mathcal{Y}_2$ domain in Eq.~\ref{l1_1} and Eq.~\ref{l2_1} with the above rule, it becomes possible for the objective to be trained on $\mathcal{X}_2, \mathcal{Y}_2$.

\textbf{2) Training with the constraint $y\perp z$.} Here we aim to convert the constraint on the loss function to a loss term using the KKT condition.
 \begin{lemma} \label{lemma:variance}
The constraint in Eq. \ref{constraints} is equivalent to
$
    \text{Var}_z(\mathbf{f}(x)|z, y)=0, \  \forall y, 
    \text{Var}_y([\mathbf e_\phi(x)]_z|z, y)=0, \ \forall z
$.
\end{lemma}
\begin{proof} [Proof of Lemma \ref{lemma:variance}]
See Appendix. \ref{proof:var}.
\end{proof}
As stated in Lemma \ref{lemma:variance}, the constraint in Eq.~\ref{constraints} is equivalent to $\text{Var}_z(\mathbf{f}(x)|z, y)=0, \  \forall y, 
\text{Var}_y([\mathbf e_\phi(x)]_z|z, y)=0, \ \forall z$. During the optimization process, we employ the Karush-Kuhn-Tucker (KKT) conditions to incorporate this constraint as an additional penalty term of the objective function. Then we write the summation form as another loss term as:
\begin{equation*}\label{penalty_term}
\begin{aligned}
\mathcal{L}_4=&\sum_{x,y\in\mathcal{X}_2,\mathcal{Y}_2} \text{Var}_z(\mathbf{f}(x)|z, y) +  
\text{Var}_y([\mathbf e_\phi(x)]_z|z, y) \\
=&n_{sample}\cdot \mathbb{E}_{p_2(y)}\mathbb{E}_{p(z)}\text{Var}_z(\mathbf{f}(x)|z, y)+\text{Var}_y([\mathbf e_\phi(x)]_z|z, y)\\
=&n_{sample}\cdot \mathbb{E}_{p_2(y)}\mathbb{E}_{p(z)}\text{Var}_z(\mathbf{f}(\mathbf{g}_\theta(z,\mathbf{m}_\gamma(y,z)),y)|z, y)+\text{Var}_y([\mathbf e_\phi(\mathbf{g}_\theta(z,\mathbf{m}_\gamma(y,z)),y)]_z|z, y)
\end{aligned}
\end{equation*}  
where $[\cdot]_z$ denotes we extract $z$ from $[z,w]$.
Then the overall objective can be expressed as:
\begin{equation}\label{overallobjective_training}
\begin{aligned}
&\mathcal{L}=\mathcal{L}_1+\alpha\mathcal{L}_2+\beta\mathcal{L}_3+\xi \mathcal{L}_4 \\
 s.t.,\  &\mathbf C_i(z, w, x, y)=0,  \quad \forall \ i\in\{1,...,N_c\}
\end{aligned}
\end{equation} 
Here $\mathbf C_0$ is not considered because it has been enforced with $\mathcal{L}_4$.

\textbf{3) Optimization Strategy.} We devised a strategy to optimize our objective with both seen and unseen data. Our proposed strategy has three high-level logics: (a) Directly enforcing property control by minimizing the error between desired properties and calculated properties of generated data. (b) Directly enforcing the disentangle term $\mathcal{L}_4$ by minimizing the variance of calculated $y$ with a group of data generated with the same desired $y$ and different $z$; and similarly, minimizing the encoded $z$ of a group of generated data with same $z$ and different $y$. (c) All the generated data can be labeled with its true $y$, then can be used in the following training, yielding an infinite training dataset. An algorithm of our proposed optimization algorithm is shown in Alg.~\ref{alg}. The training procedure of our proposed framework is also illustrated in Fig.~\ref{modelfig}.

\section{Experiments}
\subsection{Experiment Setting}
\textbf{Datasets.} The proposed framework and comparison base models are evaluated on three datasets from two domains, including image datasets and molecule datasets: Dsprites, 3DShapes, and QM9. The details can be found in Appx.~\ref{append:dataset}.

\textbf{Base Generators and Model Backbones.} We selected four models, CondVAE \cite{condvae1, condvae2}, CSVAE \cite{csvae}, Semi-VAE \cite{semivae}, and PCVAE \cite{pcvae} as the base generators for our framework. For experiments on image datasets, we employed Convolutional Neural Networks (CNNs) as the encoder and decoder. For molecule data, we employed GF-VAE \cite{ma2021gf} as the VAE backbone. 
\subsection{Quantitative Evaluation}
\textbf{Quantitative evaluation of property control.} The error of property values of the generated samples is directly calculated with the desired value $y$ and the actual value of generated data $f(x)$. The Mean Squared Error (MSE) is employed as the error metric. Results on dSprites, Shapes3D, and QM9 datasets are shown in Table \ref{result-dsprites}, Appx.~\ref{append:3d}, and Appx.~\ref{append:qm9}, correspondingly. Our model significantly improved property control for both in-distribution and out-of-distribution data, with noteworthy reductions in the MSE by over 90\% on dSprites and 60\% on Shapes3D. QM9 dataset improvements were lesser due to challenges in controlling the large latent space of the VAE backbone used for molecule data. Regardless, our model consistently surpassed all baselines in generating samples with properties closely matching the desired ones.

\textbf{Quantitative evaluation of disentanglement ensuring.}
Following the methodology proposed by $\beta$-VAE \cite{higgins2016beta}, we evaluate the disentanglement ensuring by conducting interventions on the factors of variation and predicting which factor was manipulated. We computes the variance of generated property values ($y$) corresponding to unchanged latent variables ($z$) (\textit{disetg1}), during the intervention as the first disentanglement score and the variance of $z$ encoded from samples generated from different desired $y$ as the second disentanglement score (\textit{disetg2}). 

\begin{minipage}{0.48\textwidth}
Results in the same tables as property control show that our model significantly improves the disentanglement between latent space $z$ and the properties $y$ on all datasets. The results demonstrate the effectiveness of our proposed approach of ensuring the disentanglement by directly minimizing the variances, compared to the methods of minimizing the KL-divergence of the base generators.

\textbf{Quantitative evaluation of generation quality.} Following previous work \cite{pcvae, wang2022multi}, for experiments on image datasets, we quantitatively evaluate the quality of generated images using the reconstruction error and negative log-likelihood; and we use three common metrics,  validity, novelty, and uniqueness, to evaluate the quality of generated molecules. 

In our experiments, the balance between the $\mathcal{L}_1$ and $\mathcal{L}_2$ was observed to form a trade-off between generation quality and property controllability. Results indicate that the model with our framework achieved better overall performance when considering the trade-off on generation quality and controllable generation compared to the baseline model. Results of experiments conducted on the QM9 datasets are presented in Appx.~\ref{generation-quality-molecule-qm9}. These results demonstrate that our proposed model can achieve comparable performance in generation quality compared to the base generators. We create a 2D figure to illustrate this trade-off in Appx.~\ref{append:3d},
based on experiments in the Shapes3D dataset. 

\end{minipage}
\hspace{0.04\textwidth}
\begin{minipage}{0.48\textwidth}
\vspace{-8mm}
\small
\begin{algorithm}[H]
   \caption{Optimizing via data-property mutual mappings}
   \label{alg}
   \begin{flushleft}
     {\bfseries Input:} $\mathcal{D}$,  $\phi$, $\theta$, $\gamma$, $p(z)$, $p_2(y)$, $f$ total training iterations $N$, training steps each iteration for seen and unseen data $N_1$, $N_2$,  learning rate $\eta$,  \\
     {\bfseries Output:} Optimized model parameters $\phi, \theta, \gamma$
   \end{flushleft}
\begin{algorithmic}[1]
   \FOR{$t=1$ {\bfseries to} $N$}
       \STATE $\nabla_\phi\gets 0$
       \STATE $\nabla_\theta\gets 0$
       \STATE $\nabla_\gamma\gets 0$ 
       
       \FOR{$n=1$ {\bfseries to} $N_1$} 
           \STATE pick $x, y$ from $\mathcal{D}$ 
           \STATE Calculate $\mathcal{L}_{(\mathcal X_1,\mathcal Y_1)}$ with $x, y$
           \STATE $\nabla_\phi \gets \nabla_\phi+ \nabla_\phi\mathcal{L}_{(\mathcal X_1,\mathcal Y_1)}$ 
           \STATE $\nabla_\theta \gets \nabla_\theta+ \nabla_\theta\mathcal{L}_{(\mathcal X_1,\mathcal Y_1)}$ 
           \STATE $\nabla_\gamma \gets \nabla_\gamma + \nabla_\gamma\mathcal{L}_{(\mathcal X_1,\mathcal Y_1)}$ 
           
       \ENDFOR
       \FOR{$n=1$ {\bfseries to} $N_2$}
            \STATE sample $y,z$ from $p_2(y)$ and $ p(z)$
            \STATE Calculate $\mathcal{L}_{(\mathcal X_2,\mathcal Y_2)}$ with $y, z$
            \STATE $\nabla_\phi \gets \nabla_\phi+ \nabla_\phi\mathcal{L}_{(\mathcal X_2,\mathcal Y_2)}$ 
           \STATE $\nabla_\theta \gets \nabla_\theta+ \nabla_\theta\mathcal{L}_{(\mathcal X_2,\mathcal Y_2)}$ 
           \STATE $\nabla_\gamma \gets \nabla_\gamma + \nabla_\gamma\mathcal{L}_{(\mathcal X_2,\mathcal Y_2)}$ 
           \STATE $x\gets \mathbf g_\theta(z, \mathbf m_\gamma(y,z))$
           \STATE $\mathcal D=\mathcal D\cup(x,f(x))$
       \ENDFOR
   
       \STATE $\phi \gets \phi - \eta \cdot \nabla_\phi$
       \STATE $\theta \gets \theta - \eta \cdot \nabla_\theta$
       \STATE $\gamma \gets \gamma - \eta \cdot \nabla_\gamma$
   \ENDFOR
\end{algorithmic}
\end{algorithm}
\end{minipage}

\begin{table*}[!tb]
    \caption{Our model compared to state-of-the-art methods on \textbf{dSprites} dataset according to MSE between desired properties and measured properties of generated data.}
    \label{result-dsprites}
    \small
    \begin{tabular}{l *{10}{S}}
    \toprule
    \multirow{2}{*}{Method} & \multicolumn{5}{c}{In-distribution} & \multicolumn{5}{c}{Out-of-distribution} \\
    \cmidrule(lr){2-6} \cmidrule(lr){7-11}
    & {\emph{Size}} & {\emph{x Pos}} & {\emph{y Pos}} & {\emph{Disetg1}} & {\emph{Disetg2}} 
    & {\emph{Size}} & {\emph{x Pos}} & {\emph{y Pos}} & {\emph{Disetg1}} & {\emph{Disetg2}} \\
    \midrule
    CondVAE & 10.32 & 33.49 & 31.98 & 8.46 & 101.82 & 37.45 & 60.41 & 86.79 & 9.44 & 5.56 \\
    CondVAE\(+\)\textbf{Ours} & 0.13 & 0.14 & 0.14 & 0.002 & 0.99 & 1.02 & 2.93 & 3.95 & 0.004 & 0.02\\[5pt]
    CSVAE & 13.44 & 31.79 & 39.19 & 7.45 & 629.99 & 7.24 & 26.81 & 239.21 & 6.26 & 55.27 \\
    CSVAE\(+\)\textbf{Ours} & 0.07 & 0.09 & 0.06 & 0.002 & 1.45 & 0.72 & 4.08 & 4.34 & 0.003 & 0.62\\[5pt]
    Semi-VAE & 10.99 & 25.67 & 29.70 & 5.99 & 121.54 & 59.24 & 49.85 & 16.49 & 5.64 & 3.89 \\
    Semi-VAE\(+\)\textbf{Ours} & 0.19 & 0.11 & 0.14 & 0.01 & 0.002 & 1.77 & 0.93 & 1.75 & 0.37 & 0.01\\[5pt]
    PCVAE & 10.14 & 43.14 & 36.93 & 3.67 & 297.43 & 0.80 & 37.89 & 316.16 & 2.60 & 18.65 \\
    PCVAE\(+\)\textbf{Ours} & 0.09 & 0.06 & 0.05 & 0.01 & 0.01 & 0.61 & 7.52 & 15.83 & 0.002 & 0.002\\
    \bottomrule
    \end{tabular}
\end{table*}

\begin{table*}[!tb]
    \centering
    \small
    \caption{Our model compared to state-of-the-art methods on \textbf{QM9} dataset according to MSE between desired properties and measured properties of generated data.}
    \label{result-qm9}
    \resizebox{\textwidth}{!}{%
    \begin{tabular}{lcccccccc}
    \toprule
    \multirow{2}{*}{Method} & \multicolumn{4}{c}{In-distribution} & \multicolumn{4}{c}{Out-of-distribution} \\
    \cmidrule(lr){2-5} \cmidrule(lr){6-9}
    & \emph{NumBonds} & \emph{MolWeight} & \emph{Disetg1} & \emph{Disetg2} 
    & \emph{NumBonds} & \emph{MolWeight} & \emph{Disetg1} & \emph{Disetg2} \\
    \midrule
    Semi-VAE & 22.01 & 442.24 & 2.00 & 7.70 & 14.68 & 160.81 & 1.70 & 5.64 \\ 
    Semi-VAE+\textbf{Ours} & 16.83 & 223.04 & 0.29 & 1.25 & 10.87 & 144.12 & 0.63 & 1.61\\[4pt]
    PCVAE & 20.57 & 585.46 & 1.81 & 651.19 & 17.78 & 273.39 & 2.08 & 234.02\\
    PCVAE+\textbf{Ours} & 18.09 & 547.45 & 0.49 & 1.13 & 8.03 & 222.77 & 0.33 & 2.01\\
    \bottomrule
    \end{tabular}
    }
\end{table*}

\begin{wrapfigure}{r}{0.5\textwidth}\vspace{-1.2cm}
    \begin{center}
        \hspace{-7mm}
        \includegraphics[width=0.55\textwidth]{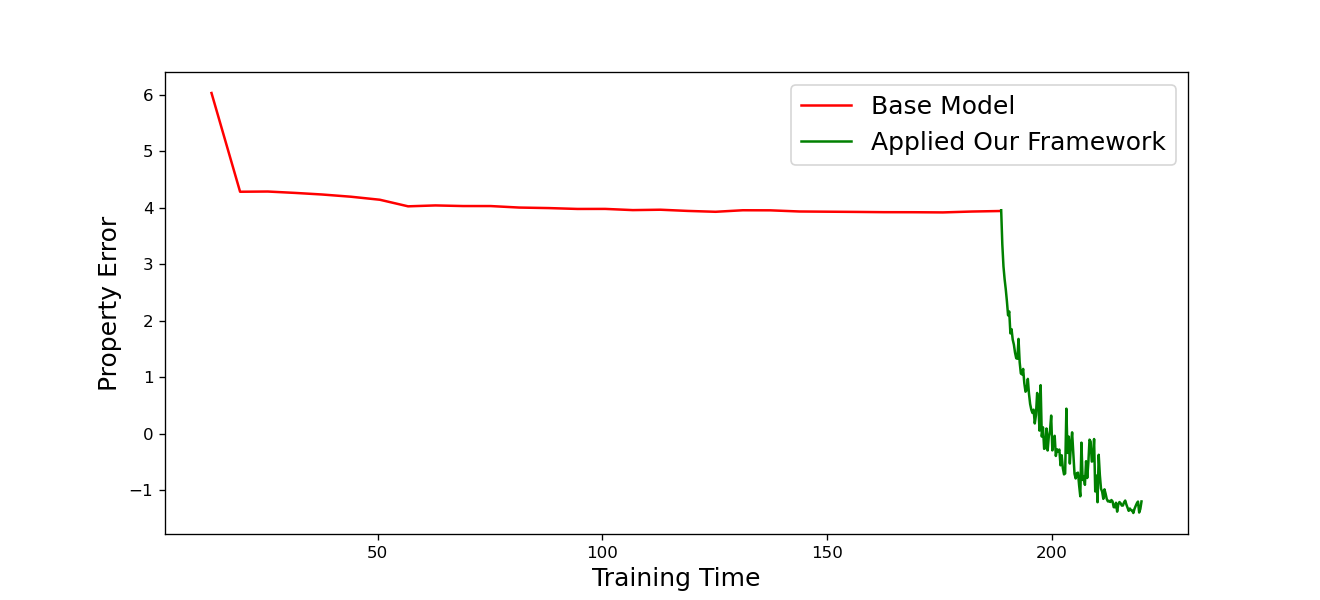}
    \end{center}\vspace{-0.4cm}
    \caption{Decrease curve of the logarithm of property error MSE with training time.\vspace{-0.4cm}}
    \label{fig:time}
\end{wrapfigure}


\textbf{Quantitative evaluation of training time.}
Additionally, a comparative evaluation of the framework's training time was conducted with the base models. The results are presented in Figure \ref{fig:time}. The findings demonstrate that utilizing our framework to optimize a pre-trained base model will incur a small additional training time compared to the time for training the base model from scratch.

\subsection{Qualitative Evaluation}


\begin{wrapfigure}{r}{0.5\textwidth}
    \vspace{-1.3cm}
    \begin{center}
        \includegraphics[width=0.5\textwidth]{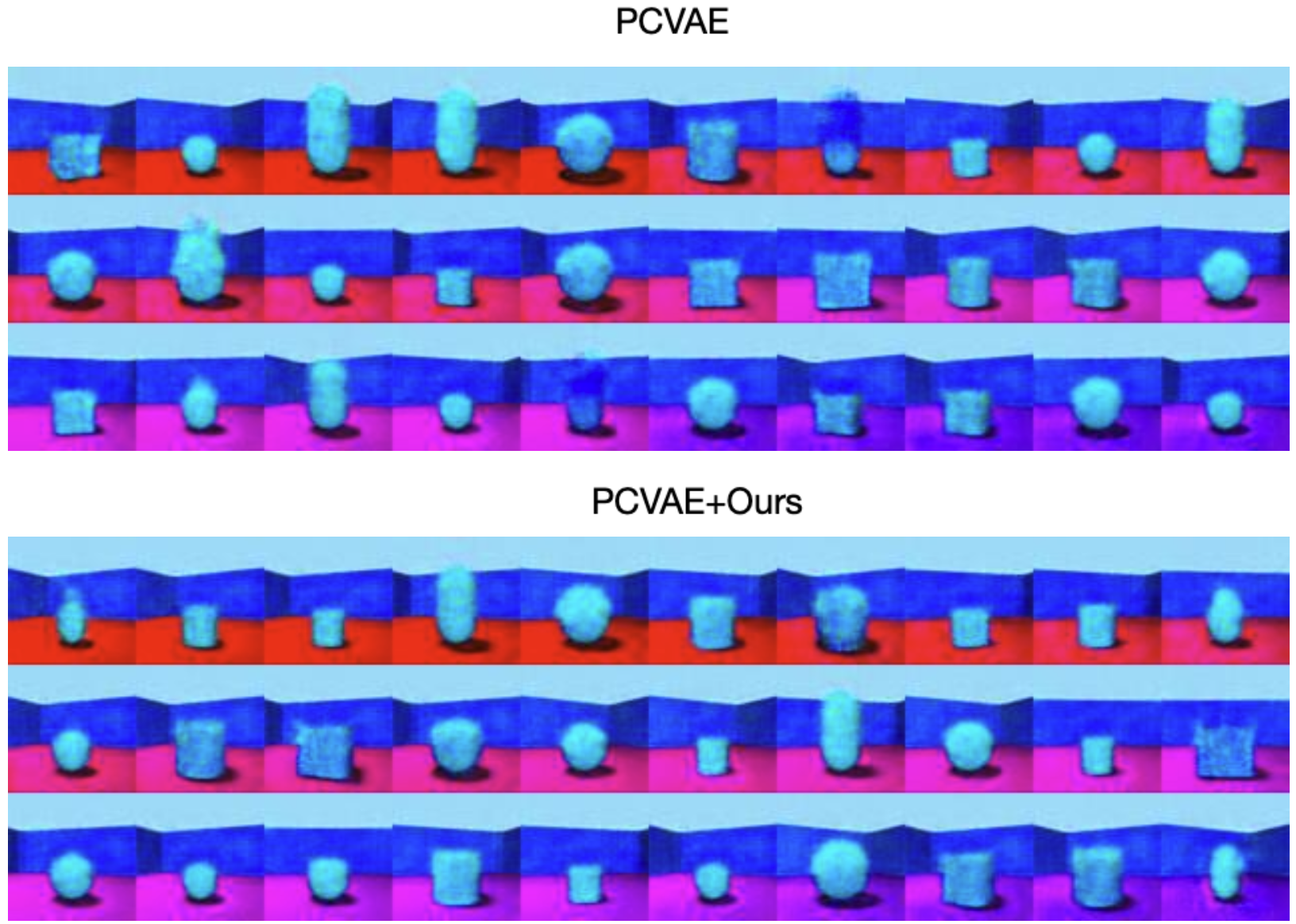}
    \end{center}\vspace{-0.2cm}
    \caption{Generated images in the OOD setting. The floor colors are discrete values in the dataset, our method can give better interpolations of colors compared to base generators. }
    \label{fig:ood}
\end{wrapfigure}

\textbf{Qualitative evaluation of OOD generation quality.} 
In Fig.~\ref{fig:ood}, we tested the OOD generation performance of our framework by interpolating between discrete floor colors in the training set. Results show that our framework can give smoother interpolation colors between the colors in the training set, demonstrating the OOD generation power and also showing the generation quality is not decreased.
\begin{wraptable}[6]{r}{0.53\textwidth}
    \centering
    \small
    \vspace{-2cm}
    \caption{Results of ablation study.}
    \begin{adjustbox}{width=0.5\textwidth}
    \begin{tabular}{p{2.7cm}p{0.7cm}p{0.7cm}p{0.7cm}p{0.7cm}p{0.7cm}}
    \toprule
    Methods & Size & x Pos & y Pos & Disetg1 & Disetg2 \\
    \midrule
    PCVAE & 10.14 & 43.14 & 36.93 & 3.67 & 297.4\\
    PCVAE-\textbf{Ours-1} & \textbf{0.03} & \textbf{0.02} & \textbf{0.03} & 2.80 & 310.27\\
    PCVAE-\textbf{Ours-2} & 10.08 & 25.32 & 30.74 & 3.07 & 238.8\\
    PCVAE-\textbf{Ours} & 0.09 & 0.06 & 0.05 & \textbf{0.01} & \textbf{0.01}\\
    \midrule
    PCVAE-OOD & 0.80&37.89&316.16&2.60&18.65\\
    PCVAE-\textbf{Ours-3}&1.27& 44.38&189.93&0.01&0.02 \\
    PCVAE-\textbf{Ours-OOD} & 	0.61&7.52&15.83&0.002&2e-5\\
    \bottomrule
    \end{tabular}
    \end{adjustbox}
    \label{ablation}
\end{wraptable}
\vspace{1mm}
\textbf{Qualitative evaluation of property control.}
As shown in Appendix \ref{append:gen-qm9}. We set the desired property as 7 bonds and 50 total atomic number. Compared to the base generator, our framework produced molecules with atomic numbers closer to 50. While Semi-VAE generated molecules with 8-11 bonds, our model created molecules with 6-8 bonds.
\subsection{Ablation Study}
To evaluate the effectiveness of the disentanglement term in our objective function and the iterative training approach, we performed an ablation study on three versions of our framework:
1) Ours-1: Removed the disentanglement terms during iterative training.
2) Ours-2: Eliminated the iterative training, applying the property control only to the base generator.
3) Ours-3: Retained both the iterative training and disentanglement terms but excluded the OOD component during iterative training. Specifically, we selected in-distribution values for generating additional data in iterative training and assessed it in an OOD context. (Results only compared with PCVAE-OOD and PCVAE-OOD-Ours.) 
This study used the dSprites dataset with PCVAE as the base generator. The results are shown in Table \ref{ablation}. The findings indicate that our model's enhanced property controllability is largely due to the iterative training approach, which directly boosts property controllability through back-propagation with property error. Excluding OOD additional data even deteriorates the MSE results for properties, as training with in-distribution data narrows the model's focus to the in-distribution range. This confirms the importance of the OOD component in our framework. 
\vspace{-2mm}
\section{Conclusion}
\vspace{-2mm}
In this paper, we attempt to tackle several challenges in property controllable data generation by proposing a general framework with its optimization methods. We proposed our framework to ensure the precision of both in-distribution and out-of-distribution property control as well as the disentanglement between properties to control and latent variables irrelevant to property control, with a procedure to effectively optimize our objective. Comprehensive experiments demonstrate the effectiveness of our proposed framework.

\bibliography{iclr2024_conference}
\bibliographystyle{iclr2024_conference}
\newpage
\appendix
\section{Examples of incorporating existing VAE-based controllable generators with our proposed framework.}\label{proof:base}

Here we show that our proposed framework can be applied to various VAE-based controllable generators, with Cond-VAE, SemiVAE, CSVAE, and PCVAE as examples. \\
For Cond-VAE, the extra assumption is $y=w$, thus an additional constraint is
\begin{equation}
C_1: y-w=0    
\end{equation}
For SemiVAE, the extra assumptions are $z\perp y$ and $y=w$, thus the additional constraints are 
\begin{equation}
\begin{aligned}
&C_1: p(z,y)-p(z)p(y)=0 \\
&C_2: y-w=0
\end{aligned}
\end{equation}
For CSVAE, the extra assumptions are $z\perp w$ and $(x|w) \perp y$, thus an additional constraint is
\begin{equation}
\begin{aligned}
    &C_1: p\left(w,z\right)-p(w)p(z)=0 \\
    &C_2: p\left((x|w),y)\right)-p(x|w)p(y)=0
\end{aligned}
\end{equation}
For PCVAE, the extra assumptions are $z\perp w$ and $x \perp y|w$, thus an additional constraint is
\begin{equation}
\begin{aligned}
    &C_1: p\left(w,z\right)-p(w)p(z)=0 \\
    &C_2: p\left((y|w),x)\right)-p(y|w)p(x)=0
\end{aligned}
\end{equation}

\section{Proof of Lemma \ref{lemma:variance}} \label{proof:var}
\begin{proof}
Here we use $\hat y$ to denote the desired properties and $y$ means the generated $y$, which is attained by $f(x)$.
\begin{equation*}
\begin{aligned}
&\text{Var}_z(\hat y|z, y)=0, \quad \forall y \\
\iff &\mathbb{E}_z(p( \hat y|z, y)-\mathbb{E}_z p(\hat y|z,  y))^2 = 0, \quad \forall  y \\
\iff &p(\hat y|z, y)-\mathbb{E}_z p(\hat y|z, y) = 0, \quad \forall  y \\
\iff & p(\hat y|z, y)-p(\hat y| y)=0, \quad \forall  y \\
\iff & \mathbb E_{p( y)} \left(p(\hat y|z, y)-p(\hat y| y)\right)=0 \\
\iff & p(\hat y|z)-p(\hat y)=0 \\
\iff & \hat y\perp z 
\end{aligned}
\end{equation*}
Above, we demonstrated that the independent constraint $\hat y\perp z$ is equivalent to the constraint $\text{Var}_z(\hat y|z, y)=0,\  \forall \hat y$. It's the same to prove $ y\perp \hat z$ is also equivalent to $\text{Var}_y(\hat z|y,  z)=0,\  \forall z$.

In our framework, $y$ and $z$ are sampled from the preferred properties distribution $p_2(y)$ and the prior distribution $p(z)$, and the $\hat y$ and $\hat z$ values of generated data are got by $\mathbf f(x)$ and $[\mathbf e_\phi(x)]_z$, correspondingly, where $[\cdot]_z$ denotes we extract $z$ from $[z,w]$. So we can get two derived constraints which are equivalent to $y \perp z$ as:
\begin{equation}
    \text{Var}_z(\mathbf{f}(x)|z, y)=0, \quad \forall y 
\end{equation}
\begin{equation}
    \text{Var}_y([\mathbf e_\phi(x)]_z|z, y)=0, \quad \forall z
\end{equation}
Thus the proof is established.

\end{proof}

\section{Dataset Details}\label{append:dataset}
(1) The \textbf{dSprites} dataset \cite{dsprites17} contains 170,000 2D shapes generated using ground-truth independent semantic factors. The properties used in this experiment are the scale and the x and y positions (designated as $x_{pos}$ and $y_{pos}$). Their ranges in dataset are $[0.5, 1], [0,1],[0,1]$. We set the out-of-distribution ranges as $[0.3, 0.5], [-0.2,0],[1,1.2]$. The split of the dataset is 160,000 images for training and 10,000 images for testing. The number of white pixels in images, average $x$ position and average $y$ position are used to estimate the selected properties. 

(2) The \textbf{Shapes3D} dataset \cite{3dshapes18} contains 480,000 3D shapes generated using ground-truth independent semantic factors. The properties used are wall hue, object hue, and floor hue. Their values range from 0 to 1 with each step 0.1. We set the out-of-distribution ranges as continuous float numbers in the range (0, 1). We segment the images into 4 colors and use the average hue in each segmented zone to estimate the properties. The split of the dataset is 390,000 images for training and 90,000 images for testing.

(3) The \textbf{QM9} dataset \cite{ramakrishnan2014quantum, ruddigkeit2012enumeration} consists of approximately 134,000 stable small organic molecules with up to 9 atoms. Properties to control include the number of bonds and the molecule weight of a molecule. Their ranges are integers in (0, 13) and (12, 144) in the dataset. Due to the constraint of the validity of molecules, we set their out-of-distribution ranges as (15, 15) and (60, 160) with overlap to the range of properties in the dataset. We use the generated adjacency matrix and an atom mass dictionary to calculate the estimated properties. The split is 121,000 for training and 13,000 for testing.
\section{Generation Quality on QM9 Dataset} 
\label{generation-quality-molecule-qm9}

Shown in Table. \ref{table:generation-quality-molecule-qm9}.

\begin{table*}[!tb]
    \centering
    \caption{Our framework compared to state-of-the-art methods for generation quality on \textbf{QM9} dataset.}
    \label{table:generation-quality-molecule-qm9}
    \vskip 0.15in
    \resizebox{0.82\textwidth}{!}{%
    \begin{tabular}{lcccccc}
    \toprule
    \multirow{2}{*}{Method} & \multicolumn{3}{c}{In-distribution} & 
    \multicolumn{3}{c}{Out-of-distribution} \\
    & \emph{Validity} & \emph{Novelty} & \emph{Uniqueness} & \emph{Validity} & \emph{Novelty} & \emph{Uniqueness} \\
    \midrule
    Semi-VAE &  100\% & 84.71\% & 94.44\% & 100\% & 94.67\% & 99.63\% \\
    Semi-VAE\(+\)\textbf{Ours} & 100\% & 88.11\% & 86.75\% & 100\% & 97.19\% & 89.00\% \\[5pt]
    PCVAE & 100\% & 98.38\% & 96.50\% & 100\% & 98.53\% & 89.06\% \\
    PCVAE\(+\)\textbf{Ours} & 100\% & 87.27\% & 61.88\% & 100\% & 98.00\% & 87.13\% \\
    \bottomrule
    \end{tabular}
    }
\end{table*}

\section{Property Error on 3DShapes dataset} \label{append:3d}
See Table~\ref{result-3dshapes}.
\section{Property Error on QM9 dataset}\label{append:qm9}
See Table~\ref{result-qm9}.
\begin{table*}[!tb]
    \centering
    \small
    \caption{Our model compared to state-of-the-art methods on \textbf{Shapes3D} dataset according to MSE between desired properties and measured properties of generated data.}
    \label{result-3dshapes}
    \resizebox{\textwidth}{!}{%
    \begin{tabular}{lcccccccccc}
    \toprule
    \multirow{2}{*}{Method} & \multicolumn{5}{c}{In-distribution} &  
    \multicolumn{5}{c}{Out-of-distribution} \\
    & \emph{Wall} & \emph{Item} & \emph{Floor} & \emph{Disetg1} & \emph{Disetg2} 
    & \emph{Wall} & \emph{Item} & \emph{Floor} & \emph{Disetg1} & \emph{Disetg2} \\
    \midrule
    CondVAE & 73.19 & 54.54 & 74.42 & 73.84 & 4.34 & 55.87 & 62.09 & 52.84 & 75.62 & 4.76 \\ 
    CondVAE+\textbf{Ours} & 19.79 & 16.18 & 14.26 & 3.56 & 0.58 & 17.60 & 14.64 & 12.32 & 8.24 & 0.62\\[4pt]
    CSVAE & 69.56 & 79.45 & 69.14 & 0.77 & 7.43 & 65.65 & 76.31 & 71.31 & 0.88 & 7.46 \\
    CSVAE+\textbf{Ours} & 28.56 & 27.78 & 27.44 & 0.01 & 0.01 & 28.96 & 26.27 & 28.42 & 0.01 & 0.001\\[4pt]
    Semi-VAE & 41.82 & 43.52 & 55.04 & 43.83 & 7.45 & 49.05 & 38.58 & 59.17 & 44.92 & 7.46 \\
    Semi-VAE+\textbf{Ours} & 7.98 & 12.28 & 13.60 & 0.02 & 0.57 & 14.04 & 12.40 & 12.52 & 0.001 & 0.67\\[4pt]
    PCVAE  & 34.34 & 47.04 & 34.24 & 42.28 & 0.85 & 37.94 & 50.51 & 43.16 & 37.51 & 99.75 \\
    PCVAE+\textbf{Ours} & 9.78 & 16.60 & 13.94 & 0.04 & 0.44 & 10.23 & 13.96 & 15.56 & 0.01 & 0.51 \\
    \bottomrule
    \end{tabular}
    }
\end{table*}

\section{Quantitative results of generation quality on 3DShapes dataset} \label{append:tradeoff}
\begin{figure*}
    \begin{center}
        \begin{subfigure}{0.47\textwidth}
         \centering
         \includegraphics[width=\textwidth]{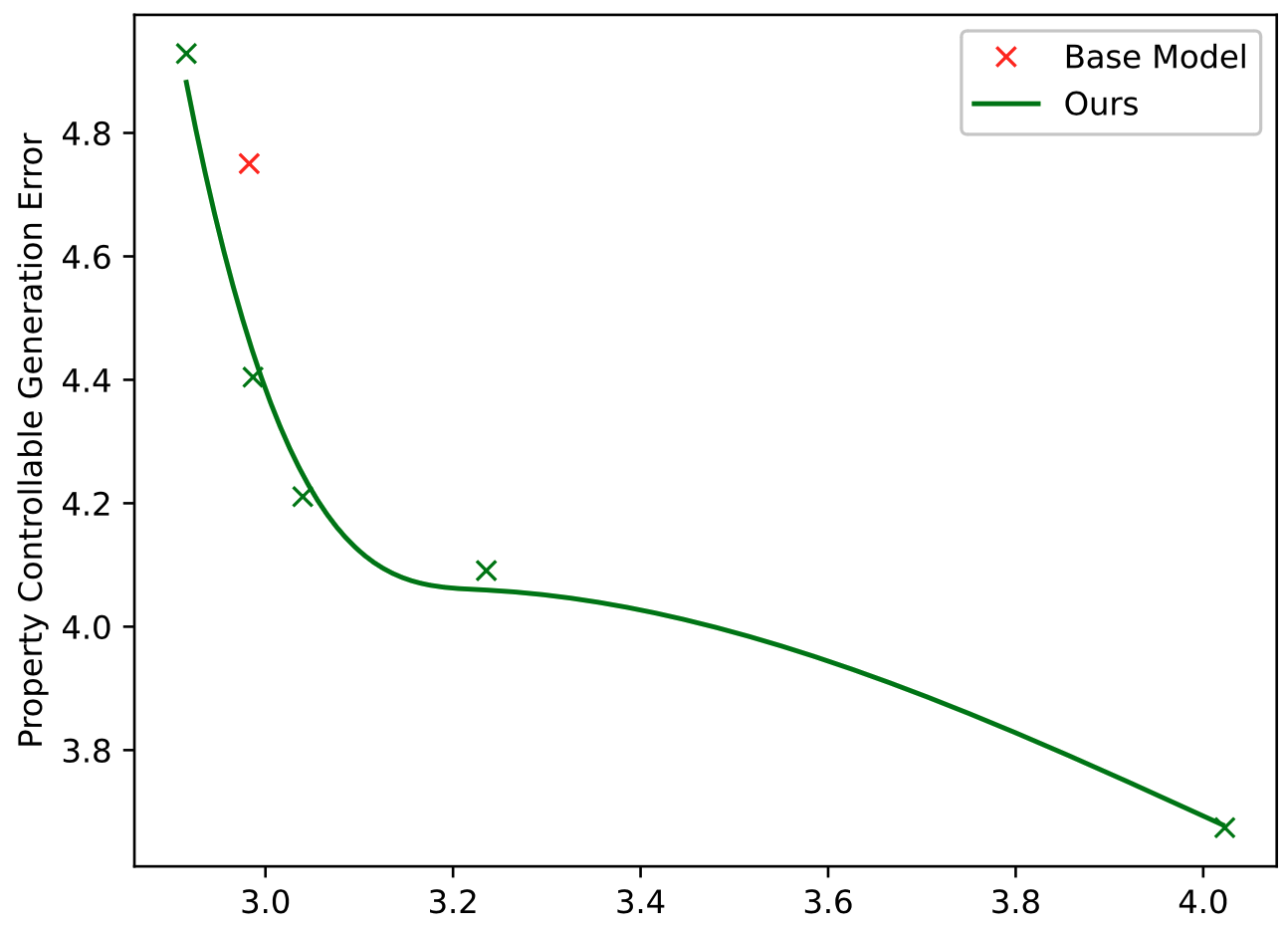}\vspace{-0.1cm}
         \caption{Reconstruction Error\vspace{-0.4cm}}
     \end{subfigure}
     \hfill
     \begin{subfigure}{0.47\textwidth}
         \centering
         \includegraphics[width=\textwidth]{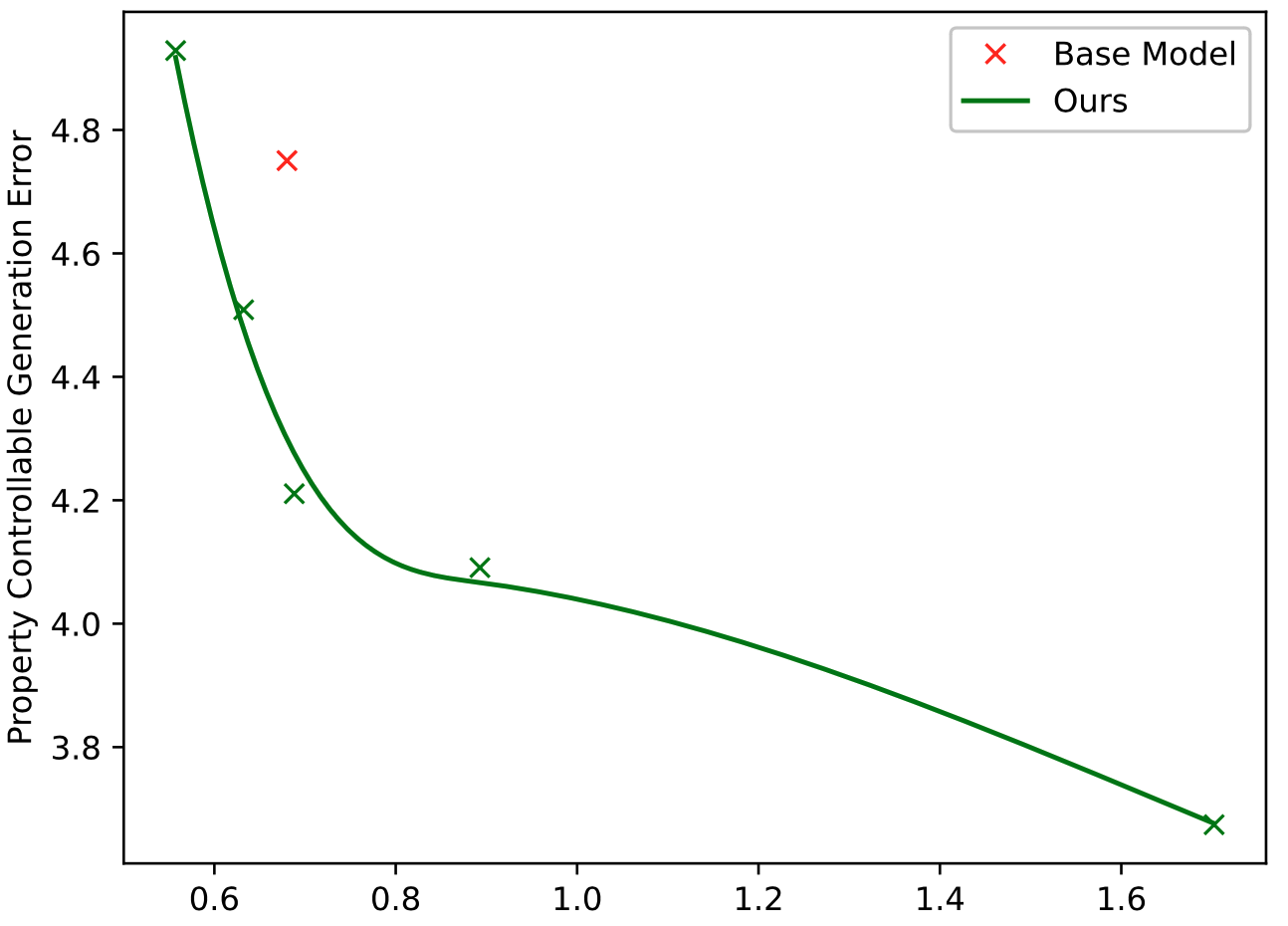}\vspace{-0.1cm}
         \caption{Negative log probability\vspace{-0.4cm}}
     \end{subfigure}
    \end{center}
    \caption{The joint comparison of generation quality and property control error. The x-axis’ are the reconstruction error and the negative log probability, and the y-axis is the property control error. Red x’s make the base generator's performance, and the green curve shows our framework's performance when adjusting the weights of the objective function. The curve demonstrates the trade-off between property controllability and generation quality. Our model yields a curve below the point given by the base model, which indicates that our model can improve both property control and generation quality at the same time by adjusting the weight of loss terms.}
    \label{fig:recon}
\end{figure*}

\section{Visualization of Generated Images}\label{append:visual-fig}
\begin{figure}
    \centering
     \begin{subfigure}{0.49\textwidth}
        \centering
         \includegraphics[width=\textwidth]{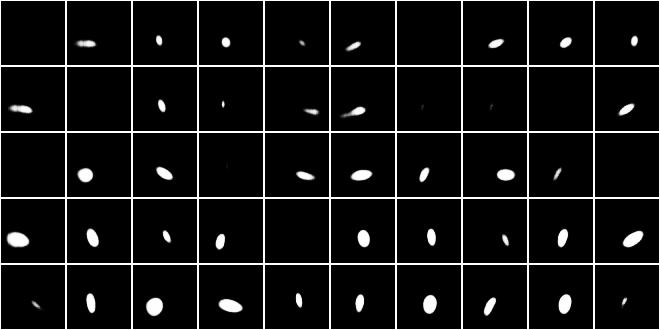}
         \caption{Out-of-distribution generation results of PCVAE\newline}
    \end{subfigure}
    \hfill
    \begin{subfigure}{0.49\textwidth}
        \centering
         \includegraphics[width=\textwidth]{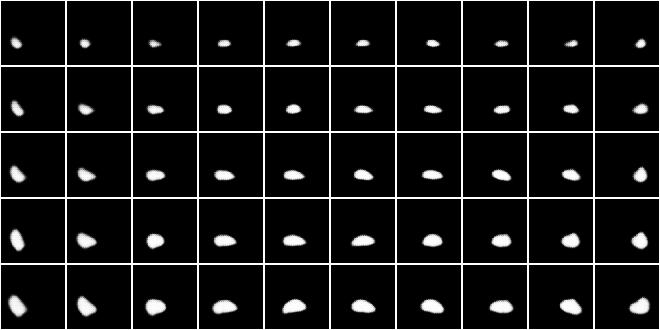}
         \caption{Out-of-distribution generation results of PCVAE applied with our framework}
    \end{subfigure}
    \caption{The comparison of controllable generated images between the base model and our framework for Out-of-distribution property values}
    \label{fig:gen-INT}
\end{figure}
Shown in Fig. \ref{fig:gen-INT}

\section{Visualization of Generated Molecules}
Shown in Fig. \ref{fig:gen-qm9}
\label{append:gen-qm9}
\begin{figure}
    \centering
     \begin{subfigure}{0.47\textwidth}
        \centering
         \includegraphics[width=1\textwidth]{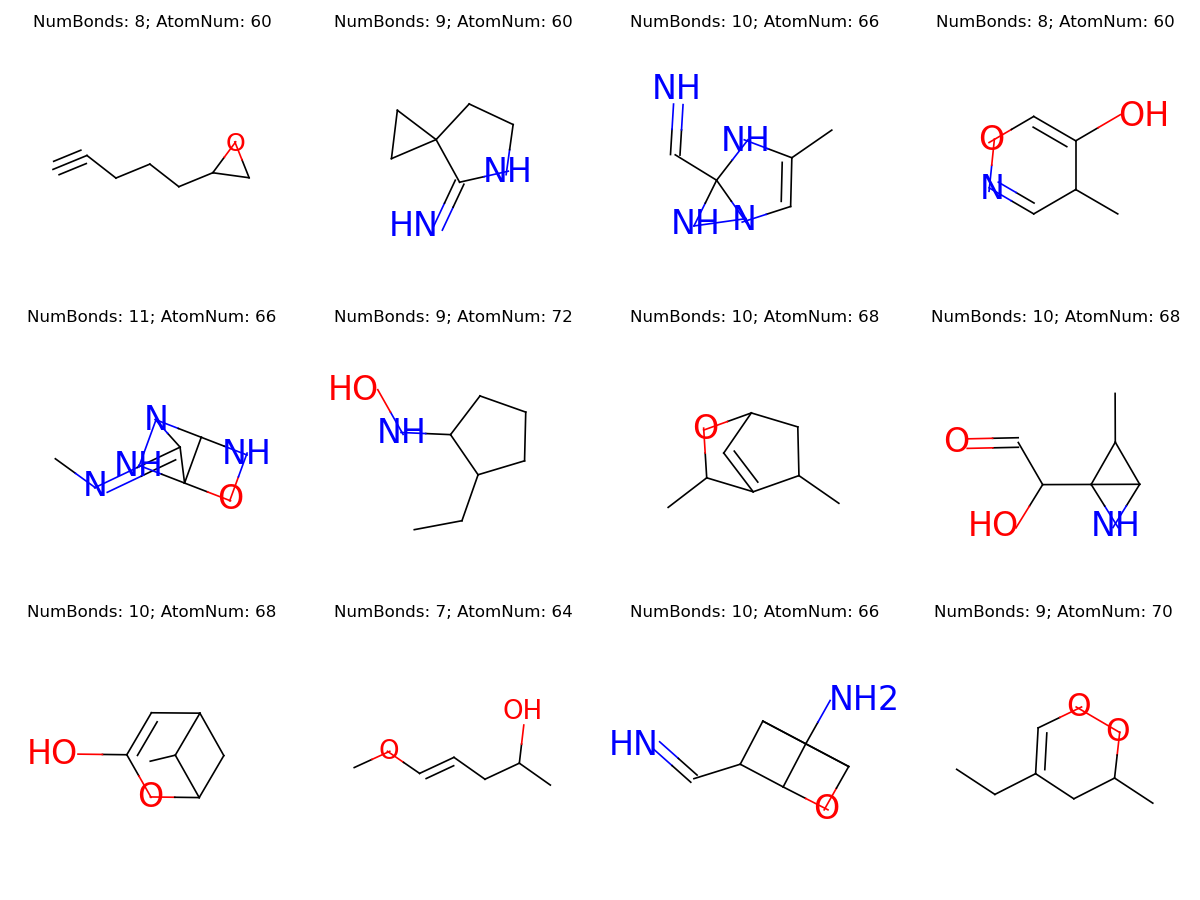}
         \caption{Generation results of Semi-VAE}
    \end{subfigure}
    \begin{subfigure}{0.47\textwidth}
        \centering
         \includegraphics[width=1\textwidth]{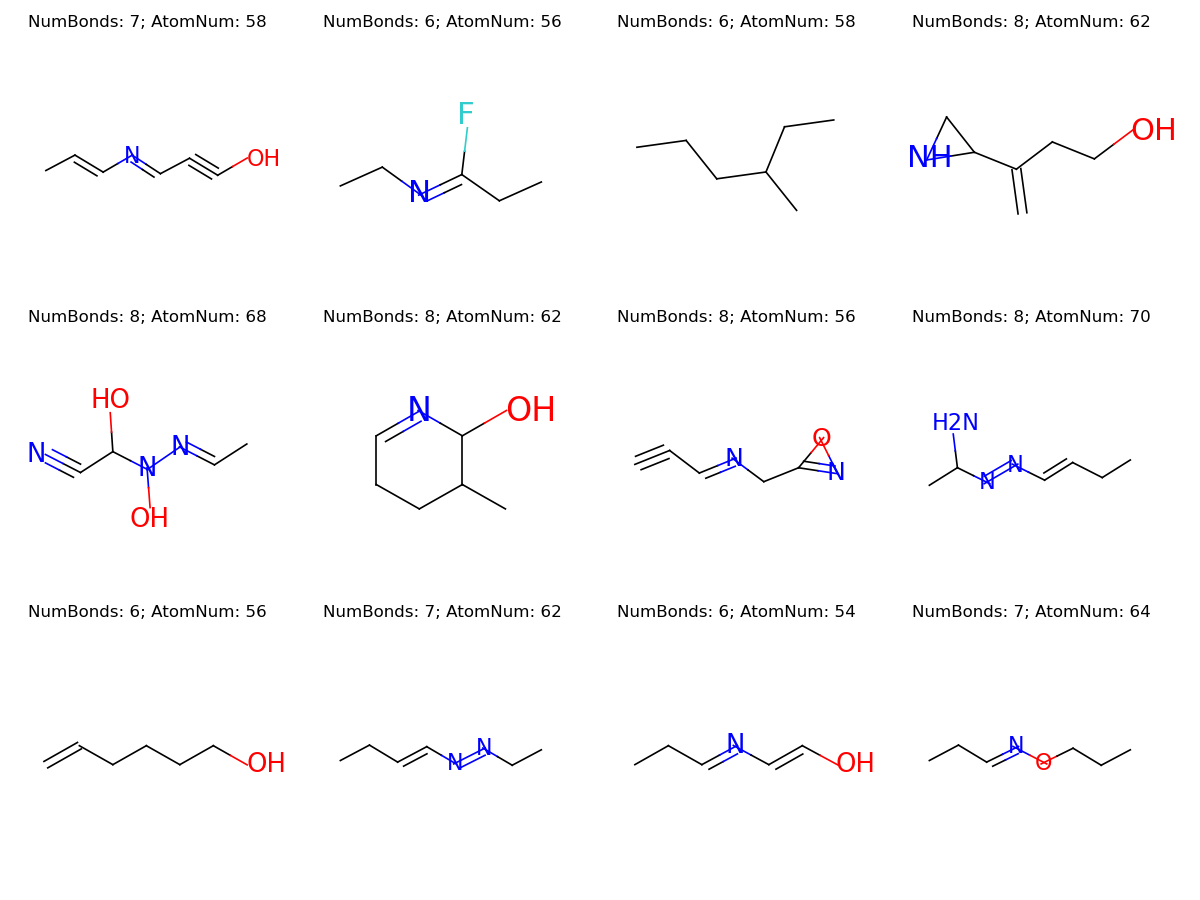}     \caption{Generation results of Semi-VAE applied with our framework}
    \end{subfigure}
    \caption{The comparison of controllable generated molecules between the base model and our framework for In-Distribution property ranges. The desired property is 7 bonds and 50 total atomic number.}
    \label{fig:gen-qm9}
\end{figure}

\end{document}